\documentclass{article}

\usepackage{amssymb,amsthm,amsmath}

\usepackage{graphicx}

\newtheorem{theorem}{Theorem}[section]
\newtheorem{lemma}[theorem]{Lemma}

\theoremstyle{definition}

\theoremstyle{remark}
\newtheorem{remark}[theorem]{Remark}

\numberwithin{equation}{section}

\begin{document}

\title{Polygon Matching and Indexing Under Affine Transformations\footnote{This work was partially supported by  CIC-UMSNH and CONACyT grants}}



\author{Edgar Ch\'avez$^\dagger$ \and Ana C. Ch\'avez-C\'aliz$^\ddagger$ \and Jorge L. L\'opez-L\'opez $^\ddagger$
\\ {\small $\dagger$Instituto de Matemáticas, UNAM, México}
\\ {\small $\ddagger$Facultad de Ciencias F\'isico-Matem\'aticas, UMSNH, México }
\\  {\tt elchavez@matem.unam.mx acchavez@fismat.umich.mx jllopez@umich.mx}  }



\maketitle

\begin{abstract}
Given a collection $\{Z_1,Z_2,\ldots,Z_m\}$ of $n$-sided polygons in the plane and a query polygon $W$ we give algorithms
to find all $Z_\ell$ such that $W=f(Z_\ell)$ with $f$ an unknown similarity transformation in time independent
of the size of the collection. If $f$ is a known affine transformation, we show how to find all $Z_\ell$ such that $W=f(Z_\ell)$ in $O(n+\log(m))$
time. 

For a pair $W,W^\prime$ of polygons we can find all the pairs $Z_\ell,Z_{\ell^\prime}$ such that $W=f(Z_\ell)$ and $W^\prime=f(Z_{\ell^\prime})$
for an unknown affine transformation $f$ in $O(m+n)$ time.

For the case of triangles we also give bounds for the problem of matching triangles with variable vertices, which is equivalent to affine matching triangles in noisy conditions.
\end{abstract}

\section{Introduction}\label{intro}

The problem of matching point sets under similarities and affine transformations has been addressed by matching sets of triangles, extracted in a consisting way, from the point sets \cite{psurvey}. In this work we propose a technique to matching arbitrary polygons under affine transforms, and as a particular case triangles in noisy conditions.

We will identify points $(x,y)$ in the plane  with the corresponding complex numbers $z=x+iy$.
A polygon in the plane will be an ordered set of points, or complex numbers. Since the set
is arbitrary, self intersection and crossings are allowed. 

An  affine transformation $f:\mathbb R^2\to\mathbb R^2$ can be (uniquely) written as
$$\label{eq:aff-rc}f(z)=\alpha z+\beta\overline{z}+\gamma$$ where $\alpha,\beta,\gamma\in\mathbb C$ and $|\alpha|^2-|\beta|^2=\det f\ne0$. Here $\overline{z}$ stands for the complex conjugated of $z$.

Given polygons $Z=(z_1,z_2,\ldots,z_n)$ and $W=(w_1,w_2,\ldots,w_n)$ with $z_\ell$ and $w_\ell$ complex numbers, the problem consist in determining if
there exists an affine transformation $f$ such that $Z=f(W)$. Since the affine transformation have three complex parameters, it is enough to find
two corresponding triples of consecutive points in both polygons. A na\"{\i}ve procedure will be to fix a triple in $Z$ and try all the shifts in $W$ to find the correspondence. This takes $O(n)$ operations.

Now assume we have a given collection of polygons $Z_1,Z_2,\ldots,Z_m$ and a query polygon $W$, and want to know which of the $Z_\ell$ are affine images of $W$. Using a sequential approach and the simple procedure above the solution can be found in $O(m n)$ operations.
In general, without an index, the complexity will depend linearly on the number of polygons in the collection; multiplied by the complexity of an individual match. We will show how to improve this complexity using invariants.


\section{Invariants}

We split the invariant in two cases because  similarity transformations (rotations, translations, scaling and compositions between them) have relevance per se.

\subsection{A Similarity Invariant for Polygons}
\label{complex}

Let $j$ and $n$ be integers, with $n\geq3$, and let $\mathfrak{p}:\{1,2,\ldots,n\}\to\{1,2,\ldots,n\}$ be a permutation. Consider functions $\varphi_{n,\mathfrak{p}},\,\varphi_{n,j}:\mathbb C^n\to\mathbb C\cup\{\infty\}$ given by 

$$\begin{array}{c}
\varphi_{\mathfrak{p}}(z_1,\ldots,z_n)=\dfrac{\sum_{k=1}^n\lambda_n^{\mathfrak{p}(k)}z_k}{\sum_{k=1}^n\lambda_n^{-\mathfrak{p}(k)}z_k},\\\\
\varphi_{n,j}(z_1,\ldots,z_n)=\dfrac{\sum_{k=1}^n\lambda_n^{jk}z_k}{\sum_{k=1}^n\lambda_n^{-jk}z_k}
\end{array}$$
where $\lambda_n=e^{2\pi i/n}$.

The function $\varphi_{\mathfrak p}$ is well-defined except on $$\mathcal N=\{(z_1,\ldots,z_n)\in\mathbb C^n\,:\,\sum_{k=1}^n\lambda^{\mathfrak{p}(k)}_nz_k=0=\sum_{k=1}^n\lambda_n^{-\mathfrak{p}(k)}z_k\},$$ but $\mathcal N$ is a $n-2$ dimensional vector subspace with null measure in $\mathbb C^n$. Similarly, $\varphi_{n,j}$ is well-defined except on a $n-2$ dimensional vector subspace.

Notice that $\varphi_{\mathfrak{p}}$ is invariant under the action of similarity transformations on polygons with $n$ vertices: $$\begin{array}{c}\varphi_{\mathfrak{p}}(\alpha z_1+\beta,\alpha z_2+\beta,\ldots,\alpha z_n+\beta)=\dfrac{\sum_{k=1}^n\lambda_n^{\mathfrak{p}(k)}(\alpha z_k+\beta)}{\sum_{k=1}^{n}\lambda_n^{-\mathfrak{p}(k)}(\alpha z_k+\beta)}\\=\dfrac{\alpha\sum_{k=1}^n\lambda_n^{\mathfrak{p}(k)}z_k+\beta\sum_{k=1}^n\lambda^{\mathfrak{p}(k)}_n}{\alpha\sum_{k=1}^{\mathfrak{p}}\lambda_n^{-\mathfrak{p}(k)}z_k+\beta\sum_{k=1}^n\lambda^{-\mathfrak{p}(k)}_n}=\varphi_{\mathfrak{p}}(z_1,z_2,\ldots,z_n).\end{array}$$

It is also clear that $\varphi_{n,j}$ is invariant under the action of similarity transformations:
$$\varphi_{n,j}(\alpha z_1+\beta,\alpha z_2+\beta,\ldots,\alpha z_n+\beta)=\varphi_{n,j}(z_1,z_2,\ldots,z_n).$$

\subsection{An Affine Invariant for Polygons} 
\label{real}

In this section, $\varphi$ denotes either $\varphi_{n,\mathfrak{p}}$ or $\varphi_{n,j}$.
Let $\zeta=\varphi(z_1,z_2,\ldots,z_n)$ and $\xi=\varphi(f(z_1),f(z_2),\ldots,f(z_n))$, where $f$ is an affine transformation. By making $x=\sum_{k=1}^n\lambda_n^{\mathfrak{p}(k)}z_k$ and $y=\sum_{k=1}^n\lambda_n^{-\mathfrak{p}(k)}z_k$ if $\varphi=\varphi_{\mathfrak{p}}$, or by making $x=\sum_{k=1}^n\lambda_n^{jk}z_k$ and $y=\sum_{k=1}^n\lambda_n^{-jk}z_k$ if $\varphi=\varphi_{n,j}$, we obtain \begin{equation}\label{quasi-conf}\dfrac{|\xi-\zeta|}{|1-\overline{\zeta}\xi|}=\dfrac{\left|\dfrac{\alpha x+\beta\overline{y}}{\alpha y+\beta\overline{x}}-\dfrac{x}{y}\right|}{\left|1-\dfrac{\overline{x}}{\overline{y}}\dfrac{\alpha x+\beta\overline{y}}{\alpha y+\beta\overline{x}}\right|}=\left|\dfrac{\beta}{\alpha}\right|.\end{equation} That is, the number $\dfrac{|\xi-\zeta|}{|1-\overline{\zeta}\xi|}$ does not depend neither on the $n$-agon $(z_1,\ldots,z_n)$ nor on the numbers $n$ and $j$, nor on the permutation $\mathfrak p$; it just depends on $f$.

Please notice that similarity invariance in section \ref{complex} can be obtained from the above taking $\beta=0$. 

\begin{remark}
The numerator and denominator involved in the definition of $\varphi_{n,j}$ are coefficients of the polygon $(z_1,\ldots,z_n)$ appearing when it has expressed in certain basis of $\mathbb C^n$, namely the basis of star-shaped polygons (see \cite{finite-fourier}, \cite[proof of Proposition 3]{yo}).
\end{remark}

\begin{remark}
Equality \eqref{quasi-conf} is analogous to a well-known result from Teichm\"uller theory: if we deform the complex structure $[T]$ of a torus by an affine transformation $f$, then the Teichm\"uller distance between $[T]$ and $[fT]$ does not depend of $T$, it just depends on $f$ (see \cite[chapter V.6]{lehto} for instance).
\end{remark}

\subsection{Shifts}\label{shift}

A polygon can be enumerated in $n$ different cyclic ways or shifts. Notice that $\varphi_{n,j}(z_1,z_2,\ldots,z_n)=\lambda_n^{2j}\varphi_{n,j}(z_2,z_3,\ldots,z_n,z_1)$, hence we have the equality of potencies
\begin{equation}\label{eq:shift}\varphi_{n,j}(z_1,z_2,\ldots,z_n)^n=\varphi_{n,j}(z_\ell,z_{\ell+1},\ldots,z_{\ell+n-1},z_{\ell+n})^n\end{equation} for any $\ell$, where the subscripts are taken mod $n$.

\begin{remark}
In what follows we use functions of the form $\varphi_{n,j}$ because $\varphi_{\mathfrak{p}}$ does not satisfies an identity of the form \eqref{eq:shift} for a general permutation $\mathfrak p$.
Each one of the instances of $\varphi_{n,j}$, where $j$ ranges in $\{1,2\,\ldots,n-1\}$, can be used in the below algorithms for matching polygons. Since they will be used as hashing functions, collisions can be avoided by using multiple $\varphi_{n,j}$ instances. 
\end{remark}

\subsection{An Index for Matching Polygons}
\label{single}

Assume that a collection of different polygons $Z_1,Z_2,\ldots,Z_m$ of $n$ edges is given. By a preprocessing step we compute pairs $(\ell, \varphi_{n,j}(Z_\ell)^n)$. Assume that a query polygon $W$ is given and that the objective is to find all the polygons in the collection such that $W=f(Z_\ell)$ for some unknown similarity transformation $f$.
Using the result from section \ref{shift}, this corresponds to all the polygons such that $\varphi_{n,j}(Z_\ell)^n=\varphi_{n,j}(W)^n$, assume there are $R$ of them.
They can be found in $O(n+R)$ operations. Please notice that this search can be implemented using hashing in time independent of the size of the collection.

If the parameters of the  affine transformation are known, then all the matching polygons such that $W=f(Z_\ell)$ can be easily found.
Using the result from section \ref{real} we compute $\zeta=\varphi_{n,j}(W)$. If the polygons match, then $\dfrac{|\xi_\ell-\zeta|}{|1-\overline{\zeta}\xi_\ell|}=\left|\dfrac{\beta}{\alpha}\right|$, where $\xi_\ell=\varphi_{n,j}(Z_\ell)$. 
We can retrieve candidate objects in sublinear time using a spatial access method (e.g. kd-trees), or a metric index \cite{msurvey}. The corresponding $R$ candidate objects in the complex plane will be arranged as a circle surrounding the query polygon $W$. Using an index, e.g. kd-trees, the complexity bound of the search will be $O(n+\log(m)+R\tau)$ with $n$ the number of edges of the polygon, and $m$ the number of polygons in the collection and $\tau$ the cost of checking the matching between a pair of polygons. 

\subsection{An Index for Matching Pairs of Polygons}
\label{pair}
If the  affine transformation is unknown, the search space cannot be bounded using invariant \ref{real}. However, the problem can be solved if we have a pair of polygons. Let $W$ and $W^\prime$ be two $n$-sided query polygons, we want to obtain all polygons $Z_\ell,Z_{\ell^\prime}$ such that $W=f(Z_\ell)$ and $W^\prime=f(Z_{\ell^\prime})$ for some unknown  affine transformation $f$. We compute $\zeta=\varphi_{n,j}(W)$ and $\zeta^\prime=\varphi_{n,j}(W^\prime)$. Let $\eta_\ell=\dfrac{|\xi_\ell-\zeta|}{|1-\overline{\zeta}\xi_\ell|}$ and $\eta^\prime_\ell=\dfrac{|\xi_\ell-\zeta^\prime|}{|1-\overline{\zeta^\prime}\xi_\ell|}$
with $\xi_\ell=\varphi_{n,j}(Z_\ell)$. Abusing the notation, let $\{Z_\ell \}=\{\eta_\ell\}$ and $\{Z_\ell^{\prime}\}=\{\eta^\prime_\ell\}$. Using the result from section \ref{real}, what we need is the intersection of the two sets above $\{\eta_\ell\} \cap \{\eta^\prime_\ell\}$. The polygons in the intersection are candidates for matching. The algorithm above needs $O(n+m)$ operations. To finish the procedure, $O(R)$ match verifications of the $R$ polygons in the intersection are needed.

\subsection{Collisions}

For two unrelated polygons $W_i,W_\ell$ it is possible that they collide, i.e. $\varphi_{n,j}(W_i)^n=\varphi_{n,j}(W_\ell)^n$ without being an affine related. This increase the $R$ candidate objects to be reviewed in algorithms above. To decrease the number of candidates we can use multiple instances of $\varphi_{n,j}$. Let $\varphi_{n,1},\varphi_{n,2},\ldots,\varphi_{n,n-1}$ be the collection of instances of the affine invariant functions,  and $L_1,L_2, \ldots,L_{n-1}$ the respective lists of candidates. The true matches should be in the intersection. Notice that it is not necessary to use all the $\varphi_{n,j}$ functions, only a subset of them.

\section{Noisy Polygons}

A slightly more general setup is when there is an unknown noise function in the matching. The image of the query polygon is an  affine transformation plus noise, namely $W=(f(z_1+\Delta z_1),f(z_2+\Delta z_2),\ldots,f(z_n+\Delta z_n))$. 
We have bounded the difference for the case of triangles, this is explained below. We were not able to give a tight bound for general polygons. To simplify notation, $\lambda_3=e^{2\pi i / 3}$ and $\varphi_{3,1}$ will be denoted respectively by $\lambda$ and $\varphi$ in what follows.

\subsection{Triangles with Variable Vertex}
For the similarity class of a triangle $(z_1,z_2,z_3)$ we can choose a representative of the form $(0,1,\tau)$. The complex number $\tau=\tau(z_1,z_2,z_3)$ is determined as $\tau=(z_3-z_1)/(z_2-z_1)$. Let us define $M(\tau)=(\lambda^2+\tau)/(\lambda+\tau)$. This is the unique conformal biyection $M:\mathbb C\cup\{\infty\}\to\mathbb C\cup\{\infty\}$ such that \begin{equation}\label{cambio}\varphi(z_1,z_2,z_3)=M(\tau(z_1,z_2,z_3)).\end{equation} The order 3 `rotation'  $R(\tau)=1/(1-\tau)$ of $\mathbb C\cup\{\infty\}$, with fixed point $-\lambda^2$, corresponds to cyclic relabelling of the vertices of a triangle, because the triangles $(0,1,\tau)$, $(R(\tau),0,1)$ and $(1,R^2(\tau),0)$ are similar. We use the term `rotation' because $M\circ R\circ M^{-1}(\xi)=\lambda\xi$ is a rotation by angle $2\pi/3$.

We shall proof an estimation for perturbations of an equilateral triangle.
For a real number $0<r<\sqrt{3}/6$ we consider the polydisc \begin{equation}\label{ur}U_r=\{(z_1,z_2,z_3)\in\mathbb C^3\,:\,|z_1-\lambda|<r,|z_2-\lambda^2|<r,|z_3-1|<r\}.\end{equation} Triangles $(z_1,z_2,z_3)\in U_r$ can be seen as perturbations of the equilateral triangle $(\lambda,\lambda^2,1)$.

\begin{lemma}\label{lem:equi}
$$\begin{array}{c}|\varphi(z_1,z_2,z_3)|\leq\sqrt{\dfrac{9-4\sqrt{3}r+12r^2-(3+2\sqrt{3}r)\sqrt{9-20\sqrt{3}r+12r^2}}{9-4\sqrt{3}r+12r^2+(3+2\sqrt{3}r)\sqrt{9-20\sqrt{3}r+12r^2}}}\\=\dfrac{8}{9}\sqrt{3}r+\dfrac{32}{27}r^2+O(r^3)\end{array}$$ for any $(z_1,z_2,z_3)\in U_r$.
\end{lemma}

Here $O(r^3)$ stands for the product of $r^3$ and a convergent power series in $r$.

\begin{proof}
First we claim that if we see triangles $(z_1,z_2,z_3)\in U_r$ as triangles in the form $(0,1,\tau)$, then $\tau$ lies in the non-shaded region showed in fig. \ref{fig:2}, which will be described below. To proof the claim, note that the length of the longest edge of a triangle $(z_1,z_2,z_3)$ with vertices in $U_r$ is at most $\sqrt{3}+2r$, and the length of the shortest edge is at least $\sqrt{3}-2r$. It follows that \begin{equation}\label{eq:apol}\dfrac{\sqrt{3}-2r}{\sqrt{3}+2r}<\dfrac{|z_j-z_k|}{|z_j-z_l|}<\dfrac{\sqrt{3}+2r}{\sqrt{3}-2r}\end{equation} with $\{j,k,l\}=\{1,2,3\}$. Apollonius' theorem says that if $K>0$, then the locus $C_K=\{\tau\in\mathbb H\,:\,|\tau-1|=K|\tau|\}$ is an euclidean circle centered at the real axis. Since $(\lambda,\lambda^2,1)$ is a positive oriented triangle, we have that the imaginary part of $\tau$ is positive. Inequalities \eqref{eq:apol} imply that $\tau$ can not lay in the shaded regions of fig. \ref{fig:1}, and $\tau$ can not lay in the rotations by $R$ and $R^2$ of the shaded regions. Hence $\tau$ should lay in a curvilinear hexagon `centered' at $-\lambda^2$ (see fig. \ref{fig:2}). The claim is proved.
\begin{figure}
  \includegraphics[scale=0.9]{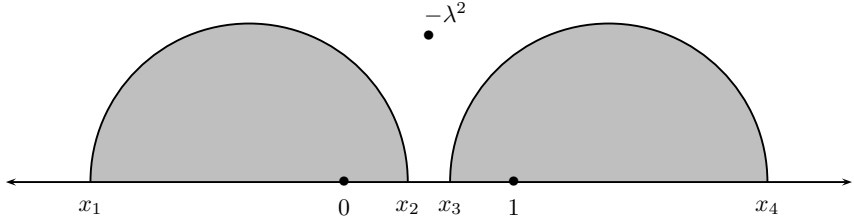}
\caption{The Apollonian circles $C_{K_1}$ and $C_{K_2}$, where $K_1=(\sqrt{3}-2r)/(\sqrt{3}+2r)$ and $K_2=1/K_1$, which have respectively diameter $x_1x_2$ and $x_3x_4$ with $x_1=\dfrac{2r-\sqrt{3}}{4r},\,x_2=\dfrac{\sqrt{3}-2r}{2\sqrt{3}},\,x_3=\dfrac{\sqrt{3}+2r}{2\sqrt{3}},\,x_4=\dfrac{\sqrt{3}+2r}{4r}$.}
\label{fig:1}
\end{figure}

\begin{figure}
  \includegraphics[scale=0.9]{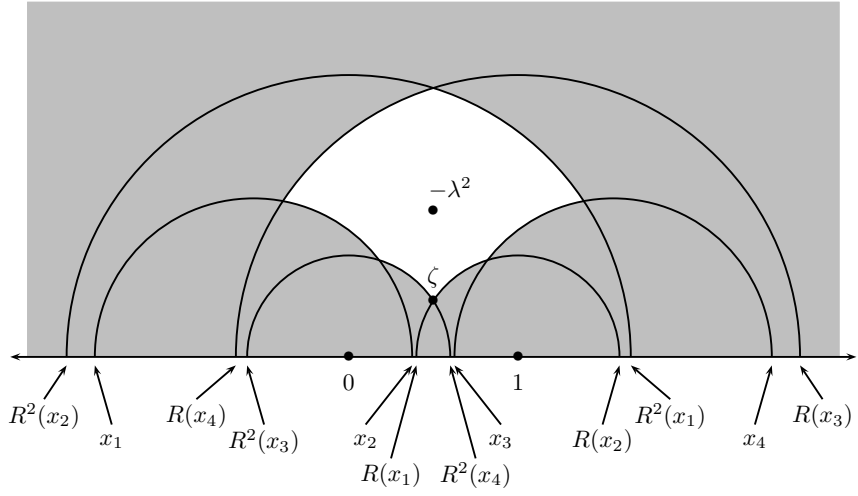}
\caption{Curvilinear hexagon centered at $-\lambda^2$ with a symmetry group $\{I,R,R^2\}$.}
\label{fig:2}
\end{figure}

To conclude the proof of the lemma, consider the vertex $$\zeta=\dfrac{1}{2}+i\dfrac{\sqrt{9-20\sqrt{3}r+12r^2}}{2(\sqrt{3}+2r)}$$ showed in fig. \ref{fig:2}. The points $R(\zeta),\,R^2(\zeta)$ are also vertices of the curvilinear hexagon. The transformationation $M$ maps the circle through $\zeta,\,R(\zeta),\,R^2(\zeta)$ onto a circle of radius $$\big|M(\zeta)\big|=\sqrt{\dfrac{9-4\sqrt{3}r+12r^2-(3+2\sqrt{3}r)\sqrt{9-20\sqrt{3}r+12r^2}}{9-4\sqrt{3}r+12r^2+(3+2\sqrt{3}r)\sqrt{9-20\sqrt{3}r+12r^2}}}$$ centered at the origin.

We have $\lim_{r\to\sqrt{3}/6^-}\zeta=1/2\in\mathbb R$, and $(0,1,\tau)$ with $\tau\in\mathbb R$ is a degenerated triangle.

The first two coefficients of the Taylor series result from a straightforward calculation.
\end{proof}

Now we deal with perturbations of a general triangle $(z_1,z_2,z_3)$.

Notice that the unique  affine transformation $f(z)=\alpha z+\beta\overline{z}+\gamma$ that takes the equilateral triangle $(\lambda,\lambda^2,1)$ to a non-degenerated triangle $(z_1,z_2,z_3)$ is \begin{equation}\label{eq:equi}\dfrac{\lambda^2 z_1+\lambda z_2+z_3}{3}z+\dfrac{\lambda z_1+\lambda^2 z_2+z_3}{3}\overline{z}+\dfrac{z_1+z_2+z_3}{3}.\end{equation} In order to give a geometric interpretation to the affine transformation $f(z)=\alpha z+\beta\overline{z}+\gamma$, we set $\alpha=|\alpha|e^{i\phi_1}$, $\beta=|\beta|e^{i\phi_2}$ and $z=re^{i\theta}$. Then $\alpha z+\beta\overline{z}+\gamma$ is equal to
$$re^{i(\phi_1+\phi_2)/2}\left[\left(|\alpha|+|\beta|\right)\cos\left(\frac{\phi_1-\phi_2}{2}+\theta\right)+i\left(|\alpha|-|\beta|\right)\sin\left(\frac{\phi_1-\phi_2}{2}+\theta\right)\right]+\gamma.$$ Therefore $f$ maps a circle of radius $r$ onto an ellipse with major axis of length $2r(|\alpha|+|\beta|)$ inclined at angle $(\phi_1+\phi_2)/2$ with the real axis, and minor axis of length $2r\big||\alpha|-|\beta|\big|$. It is convenient introduce $E_z(\rho_1,\rho_2,\theta)$ to denotes the open (i.e. without boundary) ellipse with major axis of length $\rho_1$ inclined at angle $\theta$ with the real axis, minor axis of length $\rho_2$, and centered at the point $z$. The proof of the next result is immediate from the above geometric interpretation for  affine transformations.

\begin{lemma}
Given $0<r<\sqrt{3}/6$ and a non-degenerated triangle, we set $$ae^{i\phi_1}=\dfrac{\lambda^2z_1+\lambda z_2+z_3}{3},\qquad be^{i\phi_2}=\dfrac{\lambda z_1+\lambda^2z_2+z_3}{3},$$ with $a,b>0$. Let $V=\{(w_1,w_2,w_3)\in\mathbb C^3\,:\,w_j\in E_{z_j}(\rho_1,\rho_2,\theta)\text{ for }j=1,2,3\}$ where $\rho_1=2r(a+b),\,\rho_2=2r|a-b|,\,\theta=(\phi_1+\phi_2)/2$. Then the  affine transformation such that takes $(\lambda,\lambda^2,1)$ to $(z_1,z_2,z_3)$ maps triangles with vertices in $U_r$ (see eq. \eqref{ur}) onto triangles with vertices in $V$.
\end{lemma}

By the proof of lemma \ref{lem:equi}, triangles $(z_1,z_2,z_3)\in U_r$ are similarity equivalent to triangles $(0,1,\tau)$ with $\tau$ in a region inside the circle $\mathcal C$ through $\zeta,\,R(\zeta),\,R^2(\zeta)$, where $$\zeta=\dfrac{1}{2}+i\dfrac{\sqrt{9-20\sqrt{3}r+12r^2}}{2(\sqrt{3}+2r)},$$ and this circle is mapped by $M$ to a circle centered at the origin. Circles centered at the origin are preserved under multiplication by $-1$, therefore $\mathcal C$ is preserved under $M^{-1}(-M(\tau))=(\tau-2)/(2\tau-1)$. Hence $\mathcal C$ is a circle with diameter between $\zeta$ and $$\xi=M^{-1}(-M(\zeta))=\dfrac{1}{2}+i\dfrac{3(\sqrt{3}+2r)}{2\sqrt{9-20\sqrt{3}r+12r^2}}.$$

The hypothesis in the next result are the same as in the previous lemma.

\begin{theorem}\label{teo}
Given $0<r<\sqrt{3}/6$ and a non-degenerated positively oriented triangle $(z_1,z_2,z_3)$, we set $$ae^{i\phi_1}=\dfrac{\lambda^2z_1+\lambda z_2+z_3}{3},\qquad be^{i\phi_2}=\dfrac{\lambda z_1+\lambda^2z_2+z_3}{3},
$$ with $a,b>0$. Let $V=\{(w_1,w_2,w_3)\in\mathbb C^3\,:\,w_j\in E_{z_j}(\rho_1,\rho_2,\theta)\text{ for }j=1,2,3\}$ where $\rho_1=2r(a+b),\,\rho_2=2r|a-b|,\,\theta=(\phi_1+\phi_2)/2$. Then the triangles with vertices in $V$ are similar to triangles $(0,1,\tau)$ with $\tau$ in a region inside of the ellipse $$E_{w}\left(\dfrac{2\sqrt{3}\rho(a+b)}{|z_2-z_1|},\dfrac{2\sqrt{3}\rho|a-b|}{|z_2-z_1|},\dfrac{\phi_1+\phi_2}{2}-\arg(z_2-z_1)\right)$$ where $$\begin{array}{c}w=\dfrac{1}{2}-\dfrac{\sqrt{3}(z_1+z_2-2z_3)(9+12r^2-4\sqrt{3}r)}{6(z_2-z_1)(\sqrt{3}+2r)\sqrt{9-20\sqrt{3}r+12r^2}},\\\rho=\dfrac{8\sqrt{3}r}{(\sqrt{3}+2r)\sqrt{9-20\sqrt{3}r+12r^2}}.\end{array}$$
\end{theorem}

\begin{proof}
We saw above that triangles $(z_1,z_2,z_3)\in U_r$ are similar to triangles $(0,1,\tau)$ with $\tau$ inside the circle $\mathcal C$ of radius $$\rho=\dfrac{|\xi-\zeta|}{2}=\dfrac{8\sqrt{3}r}{(\sqrt{3}+2r)\sqrt{9-20\sqrt{3}r+12r^2}}$$ centered at $(\zeta+\xi)/2$; but now we are concerned with triangles with vertices in $V$.

Let $f$ be an affine transformation of the form $f(z)=\alpha z+\beta\overline{z}+\gamma$ on $\tau$. We have $(f(0),f(1),f(\tau))=(\gamma,\alpha+\beta+\gamma,\alpha\tau+\beta\overline{\tau}+\gamma)$, which is similar to $(0,1,(\alpha\tau+\beta\overline{\tau})/(\alpha+\beta))$. That is, the affine transformation $f$, acting on the three vertices of the triangles, translates to the  affine transformation $h(\tau)=(\alpha\tau+\beta\overline{\tau})/(\alpha+\beta)$ acting on the vertex $\tau$ of $(0,1,\tau)$. We apply this remark to the transformation $f=g_2\circ g_1$ where $g_1$ and $g_2$ are the affine transformations which maps respectively $(0,1,-\lambda^2)$ to $(\lambda,\lambda^2,1)$ and $(\lambda,\lambda^2,1)$ to $(z_1,z_2,z_3)$. For this $f(z)=(\lambda-1)ae^{i\phi_1}z+(\lambda^2-1)be^{i\phi_2}\overline{z}+ae^{i\phi_1}+be^{i\phi_2}+(z_1+z_2+z_3)/3$, the corresponding $h(\tau)$ is $$h(\tau)=\dfrac{(\lambda-1)ae^{i\phi_1}}{z_2-z_1}\tau+\dfrac{(\lambda^2-1)be^{i\phi_2}}{z_2-z_1}\overline{\tau}.$$ Finally, the circle $\mathcal C$ is mapped by $h$ to the ellipse
$$E_{w}\left(\dfrac{2\sqrt{3}\rho(a+b)}{|z_2-z_1|},\dfrac{2\sqrt{3}\rho|a-b|}{|z_2-z_1|},\dfrac{\phi_1+\phi_2}{2}-\arg(z_2-z_1)\right)$$ where $w=h\left(\dfrac{\zeta+\xi}{2}\right)=\dfrac{1}{2}-\dfrac{\sqrt{3}(z_1+z_2-2z_3)(9+12r^2-4\sqrt{3}r)}{6(z_2-z_1)(\sqrt{3}+2r)\sqrt{9-20\sqrt{3}r+12r^2}}.$
\end{proof}

\begin{remark}
To obtain a bound for the image under $\varphi$ of triangles with vertices in $V$, it is sufficient apply the transformation $M$ (see equation \eqref{cambio}) to the ellipse $$E_{w}\left(\dfrac{2\sqrt{3}\rho(a+b)}{|z_2-z_1|},\dfrac{2\sqrt{3}\rho|a-b|}{|z_2-z_1|},\dfrac{\phi_1+\phi_2}{2}-\arg(z_2-z_1)\right)$$ of theorem \ref{teo}.
\end{remark}


\bibliographystyle{amsplain} 
\bibliography{affine-match-dcg.bib}

\providecommand{\bysame}{\leavevmode\hbox to3em{\hrulefill}\thinspace}
\providecommand{\MR}{\relax\ifhmode\unskip\space\fi MR }
\providecommand{\MRhref}[2]{%
  \href{http://www.ams.org/mathscinet-getitem?mr=#1}{#2}
}
\providecommand{\href}[2]{#2}
\begin{thebibliography}{1}

\bibitem{msurvey}
E.~Ch\'avez, G.~Navarro, R.~Baeza-Yates, and J.L. Marroquin, \emph{Searching in
  metric spaces}, ACM Computing Surveys \textbf{33} (2001), no.~3, 273--321.

\bibitem{finite-fourier}
J.~Chris Ficher, D.~Ruoff, and J.~Shilleto, \emph{Perpendicular polygons},
  Amer. Math. Monthly \textbf{92} (1985), no.~1, 23--37.

\bibitem{lehto}
Olli Lehto, \emph{Univalent functions and teichm{\"u}ller spaces}, Graduate
  Text in Mathematics, vol. 109, Springer-Verlag, 1987.

\bibitem{yo}
Jorge~L. L{\'o}pez-L{\'o}pez, \emph{The area as a natural pseudo-{H}ermitian
  structure on the spaces of plane polygons and curves}, Diff. Geom. Appl.
  \textbf{28} (2010), no.~5, 582--592.

\bibitem{psurvey}
Barbara Zitov{\'a} and Jan Flusser, \emph{Image registration methods: a
  survey}, Image and Vision Computing (2003), no.~21, 977--1000.

\end{thebibliography}



\end{document}